\documentclass[10pt]{cai26}
\usepackage{algorithm}
\usepackage{algpseudocode}
\usepackage{multirow}
\usepackage{graphicx}
\usepackage{xcolor}
\usepackage{colortbl}
\usepackage{amsmath}
\usepackage{amssymb}
\usepackage{booktabs}

\newtheorem{theorem}{Theorem}
\newtheorem{lemma}[theorem]{Lemma}
\newtheorem{proposition}[theorem]{Proposition}

\newtheorem{definition}{Definition}
\newtheorem{assumption}{Assumption}
\newtheorem{remark}{Remark}

\allowdisplaybreaks[4]
\makeatletter
\renewcommand\section{\@startsection{section}{1}{\z@}%
  {-1.5ex \@plus -0.5ex \@minus -.2ex}%
  {0.8ex \@plus .1ex}%
  {\bfseries\raggedright}}
\renewcommand\subsection{\@startsection{subsection}{2}{\z@}%
  {-1.5ex \@plus -0.5ex \@minus -.2ex}%
  {0.6ex \@plus .1ex}%
  {\small\bfseries\raggedright}}

\def\thm@space@setup{\thm@preskip=4pt plus 2pt minus 1pt
  \thm@postskip=4pt plus 2pt minus 1pt}
\renewenvironment{proof}[1][\proofname]{\par
  \pushQED{\qed}%
  \normalfont \topsep4\p@\@plus2\p@\relax
  \trivlist
  \item[\hskip\labelsep\itshape#1\@addpunct{.}]\ignorespaces}
{\popQED\endtrivlist\@endpefalse}
\makeatother

\setlist[itemize]{topsep=2pt, partopsep=0pt, parsep=0pt, itemsep=1pt}
\setlist[enumerate]{topsep=2pt, partopsep=0pt, parsep=0pt, itemsep=1pt}
\begin{document}
\everydisplay{%
  \abovedisplayskip=4pt plus 2pt minus 2pt
  \belowdisplayskip=4pt plus 2pt minus 2pt
  \abovedisplayshortskip=1pt plus 1pt
  \belowdisplayshortskip=2pt plus 1pt minus 1pt}
\setlength{\jot}{2pt}
\linespread{0.97}\selectfont
\pagestyle{plain}
\pagenumbering{arabic}
\setcounter{page}{1}
\def\conferenceyear{2026}
\begin{center}

\title{AdaPrivate-TS: Private Thompson Sampling for Contextual Bandits with Privacy Amplification}
\maketitle

\thispagestyle{plain}

\begin{tabular}{c}
Mohammadreza Riyazat\upstairs{\affilone,*}, Eranga Ukwatta\upstairs{\affilone}
\\[0.25ex]
{\small \upstairs{\affilone} School of Engineering, University of Guelph, Guelph, ON, Canada.} \\
\end{tabular}

\emails{
  \upstairs{*} mriyazat@uoguelph.ca
}
\vspace*{0.1in}
\end{center}

\begin{abstract}
We present AdaPrivate-TS, a differentially private contextual bandit algorithm that combines Thompson Sampling with batched zCDP composition. Our key insight is that differential privacy noise inflates the posterior covariance in a structured way---adding $\mathcal{N}(0, \sigma^2 I)$ noise to $b$ yields sampling covariance $v^2 A^{-1} + \sigma^2 A^{-2}$, which Thompson Sampling interprets as increased uncertainty rather than pure corruption. Under event-level privacy (protecting individual interactions) with stochastic contexts, we prove that the privacy cost is only $O\!\bigl(\sqrt{d}\cdot\log T/\sqrt{\rho}\bigr)$---logarithmic in $T$---because parallel composition amortizes noise across batches. Additionally, we explore privacy amplification via Poisson subsampling, which can reduce effective noise at stringent privacy budgets. Experiments on synthetic and real-world datasets demonstrate: (1) AdaPrivate-TS achieves 93-99\% of non-private performance at $\varepsilon \in [0.5, 5]$, outperforming UCB by 0.5--3.7\% and up to 18\% with tuned adaptive exploration at extreme $\varepsilon$; (2) privacy amplification provides additional 2-5\% gains at low $\varepsilon$; (3) on MovieLens and Jester, AdaPrivate-TS achieves the best overall performance among event-level baselines, dominating at $\varepsilon \geq 2$; (4) under DP-SVD private features, TS's advantage over UCB \emph{grows} to +11\%, confirming noise-as-uncertainty is not limited to reward privacy. We provide rigorous proofs for privacy guarantees under interactive zCDP composition and comprehensive evaluation including convergence curves, 12-seed CIs, and DP-SVD feature ablation.
\end{abstract}


\section{Introduction}

Online recommender systems face a fundamental tension: learning user preferences requires processing private behavioral data, yet users demand privacy guarantees. Differential privacy (DP)~\cite{dwork2006calibrating} provides rigorous protection but introduces noise that degrades learning.
While existing private contextual bandit algorithms~\cite{shariff2018differentially} treat privacy noise as pure corruption, we identify a key insight: \emph{Thompson Sampling (TS) naturally handles uncertainty through posterior sampling; DP noise inflates the posterior covariance to $v^2 A^{-1} + \sigma^2 A^{-2}$, which TS interprets as increased uncertainty rather than corruption.} This leads to \textbf{AdaPrivate-TS}, which replaces Upper Confidence Bound (UCB) with TS in the private bandit framework, achieving improved regret by converting noise into exploration.
We adopt \textbf{event-level DP} where adjacent datasets differ in one user-reward pair $(u_t, r_t)$, with actions being algorithm outputs rather than private inputs (see Definition~\ref{def:adjacency}). This aligns with Table-level DP~\cite{dwork2006calibrating} and differs from joint DP~\cite{shariff2018differentially} (protecting the entire action sequence) and user-level DP (protecting all interactions of one user). Event-level DP is appropriate when users contribute few interactions; for many-interaction users, group privacy scaling applies (see Section~\ref{sec:user_level}). Under event-level DP with stochastic contexts, the impossibility result of~\cite{shariff2018differentially} (requiring adversarial contexts and joint DP) does not apply; our batched post-processing approach enables parallel composition.

\textbf{Contributions:}
\begin{enumerate}
\item \textbf{Private Thompson Sampling:} TS with DP noise results in covariance $v^2 A^{-1} + \sigma^2 A^{-2}$ (Proposition~\ref{prop:noise_integration}); we prove a tight regret bound where the privacy cost is $O\!\bigl(\sqrt{d}\cdot\log T/\sqrt{\rho}\bigr)$, logarithmic in $T$.

\item \textbf{Privacy Amplification:} Rigorous RDP-based analysis of Poisson subsampling for bandits (Proposition~\ref{prop:amplification}), with 2--4\% gains at $\varepsilon \leq 1$.

\item \textbf{Comprehensive Experiments:} 93.5--98.7\% of non-private TS at $\varepsilon \in [0.5, 5]$, outperforming private UCB by 0.5--3.7\% (up to 18\% at extreme $\varepsilon$ with adaptive tuning) and same-model zCDP baselines by up to 6--10\% on real data. Under DP-SVD private features, TS's advantage grows to $+11\%$, and convergence curves confirm stable learning dynamics.
\end{enumerate}

\section{Background and Problem Setting}

\subsection{Contextual Bandits for Recommendation}

At round $t \in [T]$, the learner observes a candidate set $\mathcal{A}_t$ with feature vectors $\{x_{t,a} \in \mathbb{R}^d\}_{a \in \mathcal{A}_t}$, selects action $a_t \in \mathcal{A}_t$, and receives stochastic reward $r_t \in [0,1]$. We assume a linear reward model:
\begin{equation}\label{eq:linear_reward}
\mathbb{E}[r_t | a_t = a] = x_a^\top \theta^*
\end{equation}
for unknown parameter $\theta^* \in \mathbb{R}^d$ with $\|\theta^*\|_2 \leq S$.

\begin{definition}
The cumulative regret over $T$ rounds is:
\begin{equation}
R(T) = \sum_{t=1}^T \left( \max_{a \in \mathcal{A}_t} x_a^\top \theta^* - x_{a_t}^\top \theta^* \right)
\end{equation}
\end{definition}

The learner maintains sufficient statistics:
\begin{align}
A_t &= \lambda I_d + \sum_{s=1}^{t-1} x_{a_s} x_{a_s}^\top \label{eq:A_update}\\
b_t &= \sum_{s=1}^{t-1} r_s x_{a_s} \label{eq:b_update}
\end{align}
with ridge parameter $\lambda > 0$. The regularized least-squares estimate is $\hat{\theta}_t = A_t^{-1} b_t$.

\subsection{Privacy Model}

\begin{definition}\label{def:adjacency}
An interaction dataset $D = \{(u_t, r_t)\}_{t=1}^T$ records the user identity $u_t$ and reward $r_t$ at each round. Two datasets $D, D'$ are \textbf{adjacent} (written $D \sim D'$) if they differ in exactly one entry $(u_t, r_t) \neq (u_t', r_t')$ for a single index $t$, with all other entries identical. Actions $a_t$ are \emph{outputs} of the algorithm (determined by the mechanism and public features) and are not part of the adjacency relation.
\end{definition}
We use zero-concentrated DP (zCDP)~\cite{bun2016concentrated} for tighter composition:

\begin{definition}
A randomized mechanism $\mathcal{M}$ satisfies $\rho$-zCDP if for all adjacent $D \sim D'$ and all $\alpha > 1$:
\begin{equation}
D_\alpha(\mathcal{M}(D) \| \mathcal{M}(D')) \leq \rho \alpha
\end{equation}
where $D_\alpha$ is the R\'enyi divergence of order $\alpha$.
\end{definition}

\begin{lemma}\label{lem:gaussian}
For a query $f: \mathcal{D} \to \mathbb{R}^d$ with $\ell_2$-sensitivity $\Delta_2 = \max_{D \sim D'} \|f(D) - f(D')\|_2$, the mechanism $\mathcal{M}(D) = f(D) + \mathcal{N}(0, \sigma^2 I_d)$ satisfies $\rho$-zCDP with:
\begin{equation}\label{eq:gaussian_rho}
\rho = \frac{\Delta_2^2}{2\sigma^2}
\end{equation}
\end{lemma}

\begin{lemma}\label{lem:conversion}
If $\mathcal{M}$ satisfies $\rho$-zCDP, then for any $\delta > 0$, $\mathcal{M}$ satisfies $(\varepsilon, \delta)$-DP with:
\begin{equation}\label{eq:conversion}
\varepsilon = \rho + 2\sqrt{\rho \log(1/\delta)}
\end{equation}
\end{lemma}

\begin{remark}\label{rem:rho_computation}
Given a target $(\varepsilon, \delta)$-DP guarantee, we solve Eq.~\eqref{eq:conversion} for $\rho$ in closed form. Setting $u = \sqrt{\rho}$ and $c = \sqrt{\log(1/\delta)}$, Eq.~\eqref{eq:conversion} becomes $u^2 + 2cu - \varepsilon = 0$, results in:
\begin{equation}\label{eq:rho_from_eps}
\rho = \bigl(\sqrt{\varepsilon + \log(1/\delta)} - \sqrt{\log(1/\delta)}\bigr)^2
\end{equation}
The noise scale is $\sigma = \Delta_2 / \sqrt{2\rho}$. All experiments compute $\rho$ from Eq.~\eqref{eq:rho_from_eps}.
\end{remark}

\begin{lemma}\label{lem:parallel}
If mechanisms $\mathcal{M}_1, \ldots, \mathcal{M}_K$ operate on disjoint subsets of the data and each satisfies $\rho_k$-zCDP, then releasing all outputs satisfies $\max_k \rho_k$-zCDP.
\end{lemma}

\subsection{Thompson Sampling for Linear Bandits}

Thompson Sampling (TS)~\cite{thompson1933likelihood,agrawal2013thompson} maintains a posterior distribution over $\theta^*$ and selects actions by sampling from this posterior. For linear bandits with Gaussian likelihood, the posterior is:
\begin{equation}\label{eq:posterior}
\theta | \mathcal{H}_t \sim \mathcal{N}(\hat{\theta}_t, v^2 A_t^{-1})
\end{equation}
where $\mathcal{H}_t$ is the history up to time $t$ and $v > 0$ controls exploration.

TS handles uncertainty through sampling where larger posterior variance leads to more exploration, unlike UCB's deterministic bonus.

\section{Method: AdaPrivate-TS}

\subsection{Key Insight: DP Noise as Posterior Inflation}

Consider adding Gaussian noise $\eta \sim \mathcal{N}(0, \sigma^2 I_d)$ to $b_t$ for privacy:
\begin{equation}
\tilde{b}_t = b_t + \eta
\end{equation}

For UCB, this noise directly corrupts the estimate. For Thompson Sampling, we characterize exactly how it affects the sampling distribution:

\begin{proposition}\label{prop:noise_integration}
Let $\eta \sim \mathcal{N}(0, \sigma^2 I_d)$ be independent privacy noise and $\xi \sim \mathcal{N}(0, v^2 A_t^{-1})$ be the TS sampling noise. The TS sample based on $\tilde{b}_t = b_t + \eta$ satisfies:
\begin{equation}\label{eq:covariance}
\tilde{\theta}_t = A_t^{-1}\tilde{b}_t + \xi \sim \mathcal{N}\left(\hat{\theta}_t, \; v^2 A_t^{-1} + \sigma^2 A_t^{-2}\right)
\end{equation}
\end{proposition}

\begin{proof}
Let $\xi \sim \mathcal{N}(0, v^2 A_t^{-1})$ be the TS sampling noise. Then:
\begin{align}
\tilde{\theta}_t &= A_t^{-1}\tilde{b}_t + \xi = A_t^{-1}(b_t + \eta) + \xi \nonumber\\
&= \hat{\theta}_t + A_t^{-1}\eta + \xi
\end{align}
Since $\eta \sim \mathcal{N}(0, \sigma^2 I_d)$ and $\xi \sim \mathcal{N}(0, v^2 A_t^{-1})$ are independent Gaussians:
\begin{align}
\text{Cov}(\tilde{\theta}_t) &= \text{Cov}(A_t^{-1}\eta) + \text{Cov}(\xi) \nonumber\\
&= A_t^{-1} \cdot \sigma^2 I_d \cdot A_t^{-\top} + v^2 A_t^{-1} \nonumber\\
&= \sigma^2 A_t^{-2} + v^2 A_t^{-1}
\end{align}
where $A_t^{-2} = A_t^{-1} A_t^{-1}$ since $A_t$ is symmetric positive definite.
\end{proof}

\begin{remark}
The covariance is exactly $v^2 A_t^{-1} + \sigma^2 A_t^{-2}$. Since $A_t$ is symmetric positive definite, the L\"owner ordering $A_t^{-2} \preceq A_t^{-1}/\lambda_{\min}(A_t)$ holds (each eigenvalue $\lambda_i^{-2} \leq \lambda_i^{-1}/\lambda_{\min}$), giving:
\begin{equation}
\text{Cov}(\tilde{\theta}_t) \preceq \left(v^2 + \frac{\sigma^2}{\lambda_{\min}(A_t)}\right) A_t^{-1}
\end{equation}
As data accumulates, $\lambda_{\min}(A_t)$ grows, so the privacy contribution diminishes relative to the base TS variance. TS interprets this inflated covariance as increased uncertainty, exploring more when privacy noise dominates.
\end{remark}

\subsection{Privacy Amplification via Subsampling}\label{sec:amplification}

At stringent privacy budgets ($\varepsilon \leq 1$), we analyze privacy amplification via \emph{Poisson} subsampling (not without-replacement mini-batching) with rate $q \in (0,1]$ where each batch reward is included independently with probability $q$.

\textit{The subsampled mechanism operates as follows:}
(1) First include each reward $r_t x_{a_t}$ independently with probability $q$, forming $s_k^{\text{sub}} = \sum_{t \in \text{batch } k} Z_t \cdot r_t x_{a_t}$ where $Z_t \sim \text{Bern}(q)$; (2) Then add noise $\tilde{s}_k^{\text{sub}} = s_k^{\text{sub}} + \eta_k$ with $\eta_k \sim \mathcal{N}(0, \sigma^2 I_d)$; (3) Now rescale for unbiasedness $\tilde{b}_{\text{batch}} = \tilde{s}_k^{\text{sub}} / q$. Crucially, the Gaussian noise is added to the subsampled sum \emph{before} rescaling. The rescaling by $1/q$ is a deterministic post-processing step that does not affect privacy. The RDP bound in Eq.~\eqref{eq:rdp_amplification} applies to step~(2)---the composition of Poisson subsampling and Gaussian mechanism---and is preserved under the subsequent rescaling by the post-processing property of RDP. This avoids the sensitivity-rescaling cancellation described below. Poisson sampling is required for Eq.~\eqref{eq:rdp_amplification}. After rescaling, sensitivity becomes $\Delta_2/q$, giving $\rho/q^2$; combined with the $q^2$ amplification factor, the net cost is $\rho$---the gain cancels~\cite{balle2018privacy}.
The resolution uses R\'enyi divergence directly on the subsampled mechanism without factoring through sensitivity. For Poisson subsampling with rate $q$ before a Gaussian mechanism with noise $\sigma$ and sensitivity $\Delta_2 = 1$, at RDP order $\alpha \geq 2$~\cite{balle2018privacy,mironov2019renyi}:
\begin{equation}\label{eq:rdp_amplification}
\varepsilon_{\text{sub}}(\alpha) \leq \frac{1}{\alpha - 1}\log\left(1 + \binom{\alpha}{2}\frac{2q^2}{\sigma^2} + \sum_{j=3}^{\alpha}\binom{\alpha}{j}\frac{(2q^2)^{j/2}}{\sigma^j}\right)
\end{equation}
For $q\alpha/\sigma \ll 1$ (small subsampling rate or large noise), the dominant term gives:
\begin{equation}\label{eq:rdp_approx}
\varepsilon_{\text{sub}}(\alpha) \approx \frac{q^2 \alpha}{\sigma^2} = 2q^2\rho
\end{equation}
This is a factor of $2q^2$ reduction from $\rho$, achieved because the RDP analysis tracks the privacy loss random variable directly, avoiding the sensitivity-rescaling cancellation.

We evaluate Eq.~\eqref{eq:rdp_amplification} at orders $\alpha \in \{2, \ldots, 6, 8, 16, 32, 64\}$.\\
Then convert via $\varepsilon(\alpha) {=} \varepsilon_{\text{sub}}(\alpha) {+} \log(1/\delta)/(\alpha{-}1)$, and minimize over $\alpha$. Under parallel composition, total cost is $\max_k \varepsilon_{\text{sub},k}$. Optimal $\alpha$ is 16 at $\varepsilon{=}1$, 3 at $\varepsilon{=}5$; final $\varepsilon$ varies ${<}0.5\%$.
For AdaPrivate-TS-Amp, we recalibrate $\sigma$ given target $(\varepsilon, \delta)$ and rate $q$, we solve for $\sigma_{\text{amp}}$ such that $\min_\alpha[\varepsilon_{\text{sub}}(\alpha; q, \sigma_{\text{amp}}) + \log(1/\delta)/(\alpha{-}1)] \leq \varepsilon$. Since amplification reduces privacy cost, $\sigma_{\text{amp}} < \sigma_{\text{base}}$, results in less noise at the same privacy level.

\begin{proposition}\label{prop:amplification}
The Poisson-subsampled Gaussian mechanism with rate $q$ satisfies $\varepsilon_{\text{sub}}(\alpha)$-RDP at order $\alpha$ as given by Eq.~\eqref{eq:rdp_amplification}. For $q\alpha/\sigma \ll 1$, the dominant term gives $\varepsilon_{\text{sub}}(\alpha) \approx q^2\alpha/\sigma^2$, a factor $2q^2$ reduction from the base $\rho$-zCDP cost. For $\sigma \geq 1.05$ and $q \leq 0.5$, higher-order terms contribute $< 5\%$. \\
Note: the subsampled mechanism's RDP profile is not linear in $\alpha$, so we state this as an RDP (not zCDP) guarantee and convert to $(\varepsilon, \delta)$-DP per the accounting details above.
\end{proposition}

\subsection{Adaptive Exploration Decay}

With privacy noise inflating effective variance, early exploration may be excessive. We decay exploration as $v_k = v_0 \cdot \gamma^{k-1}$ ($v_0{=}1.5$, $\gamma{=}0.95$), ensuring aggressive exploration when uncertainty is high and refined exploitation as the model converges.

\subsection{Complete Algorithm}

\begin{algorithm}[t]
\caption{AdaPrivate-TS}
\label{alg:adaprivate_ts}
\footnotesize
\begin{algorithmic}[1]
\Require Features $\{x_a\}$, prior $\theta_0$, target $(\varepsilon, \delta)$, batch size $B$
\State Compute $\rho$ from $(\varepsilon, \delta)$ via Eq.~\eqref{eq:rho_from_eps}
\State Set $\sigma = 1/\sqrt{2\rho}$
\State $A \gets \lambda I_d$, $b \gets \lambda \theta_0$, batch\_b $\gets 0$, count $\gets 0$, $k \gets 1$
\For{$t = 1, \ldots, T$}
    \State Observe $\mathcal{A}_t$; clip features: $x_a \gets x_a / \max(1, \|x_a\|_2)$
    \State \textbf{Thompson Sampling:}
    \State \quad $L_A \gets \text{Cholesky}(A)$ \Comment{$L_A L_A^\top = A$}
    \State \quad $\tilde{\theta} \gets L_A^{-\top}(L_A^{-1}b + v_k \cdot \mathcal{N}(0, I_d))$
    \State \quad $a_t \gets \arg\max_{a \in \mathcal{A}_t} x_a^\top \tilde{\theta}$
    \State Observe reward $r_t \in [0,1]$
    \State $A \gets A + x_{a_t} x_{a_t}^\top$
    \State batch\_b $\gets$ batch\_b $+ r_t \cdot x_{a_t}$
    \State count $\gets$ count $+ 1$
    \If{count $\geq B$} \Comment{Batch boundary}
        \State $b \gets b + \text{batch\_b} + \mathcal{N}(0, \sigma^2 I_d)$ \Comment{Add noise}
        \State Reset batch\_b, count; $k \gets k+1$; $v_k \gets v_{k-1} \cdot \gamma$
    \EndIf
\EndFor
\end{algorithmic}
\end{algorithm}

\subsection{Privacy Analysis}

\begin{assumption}\label{ass:bounded}
Features satisfy $\|x_a\|_2 \leq 1$ (enforced via clipping) and rewards satisfy $r_t \in [0,1]$.
\end{assumption}

\begin{theorem}\label{thm:privacy}
Under Assumption~\ref{ass:bounded} with data-independent batch schedule $\{B_k\}_{k=1}^K$, Algorithm~\ref{alg:adaprivate_ts} satisfies $(\varepsilon, \delta)$-DP (event-level) for the transcript $(a_1, \ldots, a_T)$.
\end{theorem}

\begin{proof}
We formalize the mechanism as releasing $K$ independent noisy \emph{increments}, then show the full transcript is a deterministic post-processing of these releases and public data.

\textbf{Step 1:} Define the batch-$k$ sufficient statistic increment $s_k = \sum_{t \in \text{batch } k} r_t x_{a_t} \in \mathbb{R}^d$. The mechanism releases the noisy increment:
\begin{equation}\label{eq:noisy_increment}
\tilde{s}_k = s_k + \eta_k, \quad \eta_k \sim \mathcal{N}(0, \sigma^2 I_d), \quad k = 1, \ldots, K
\end{equation}
Although Algorithm~\ref{alg:adaprivate_ts} maintains a running sum $b \gets b + \tilde{s}_k$, this cumulative $b$ is a deterministic function of $\{\tilde{s}_1, \ldots, \tilde{s}_k\}$: namely $b_k = \lambda\theta_0 + \sum_{j=1}^k \tilde{s}_j$. Thus the internal mechanism releases exactly the $K$ independent noisy increments $\{\tilde{s}_k\}_{k=1}^K$.

\textbf{Step 2:} Under Assumption~\ref{ass:bounded}, changing one reward $(r_t, x_{a_t})$ in batch $k$ changes $s_k$ by at most $\|r_t x_{a_t}\|_2 \leq 1$. Thus $\Delta_2 = 1$, and by Lemma~\ref{lem:gaussian} each noisy release $\tilde{s}_k$ satisfies $\rho$-zCDP with $\sigma = 1/\sqrt{2\rho}$.

\textbf{Step 3:} Action $a_t$ in batch $k$ is determined by: (i) public features $\{x_{t,a}\}$; (ii) the design matrix $A_t = \lambda I + \sum_{s < t} x_{a_s} x_{a_s}^\top$ (computed from public features and past actions); (iii) the cumulative noisy reward vector $b_{k-1} = \lambda\theta_0 + \sum_{j=1}^{k-1} \tilde{s}_j$ (from \emph{already-released} increments of prior batches); and (iv) fresh TS randomness $\xi_t$. Crucially, $a_t$ does \emph{not} depend on within-batch rewards $\{r_s\}_{s \in \text{batch } k}$, since $b$ updates only at batch boundaries.

\textbf{Step 4:} Each rating belongs to exactly one batch. For adjacent $D, D'$ differing in batch~$j$, although later batches adapt to prior releases, this does not increase privacy cost: for any realization of prior outputs $h_{<k}$, the conditional distribution of $\tilde{s}_k | h_{<k}$ is identical under $D$ and $D'$ when $k \neq j$ (since batch~$k$'s data is unchanged), with $D_\alpha(P_j(\cdot|h_{<j}) \| Q_j(\cdot|h_{<j})) \leq \rho\alpha$ for $k{=}j$ by Step~2. By the chain rule for R\'enyi divergence: $D_\alpha(P_{1:K} \| Q_{1:K}) \leq \rho\alpha$, giving $\rho$-zCDP.

\textbf{Step 5:} Given $\{\tilde{s}_k\}_{k=1}^K$ and public data (features, TS randomness), the entire action transcript $(a_1, \ldots, a_T)$ is a deterministic function---each $a_t$ is computed from $b_{k-1}$, $A_t$, and $\xi_t$ as in Step~3. By the post-processing property of zCDP, the transcript inherits $\rho$-zCDP, converting to $(\varepsilon, \delta)$-DP via Lemma~\ref{lem:conversion}.
\end{proof}

\subsection{Regret Analysis}

\begin{assumption}\label{ass:stochastic}
At each round $t$, the candidate set $\mathcal{A}_t$ contains arms whose features are drawn i.i.d.\ from a distribution satisfying $\lambda_{\min}(\mathbb{E}[x_a x_a^\top]) \geq \lambda_0 > 0$ and $\|x_a\|_2 \leq 1$.
\end{assumption}

\begin{theorem}\label{thm:ts_regret}
Under Assumptions~\ref{ass:bounded} and \ref{ass:stochastic}, AdaPrivate-TS with $K$ uniform batches of size $B = T/K$ achieves, for any $\delta_r \in (0,1)$, with probability at least $1 - \delta_r$:
\begin{equation}\label{eq:regret_bound}
R(T) \leq \underbrace{C_1 d\sqrt{T \log(T/\delta_r)}}_{\text{TS + staleness}} + \underbrace{\frac{C_2 \sigma \sqrt{d} \log(T/\delta_r)}{\lambda_0}}_{\text{privacy (amortized)}}
\end{equation}
where $C_1 = O(vS)$, $C_2 = O(S)$, and $S = \|\theta^*\|_2$. In expectation:
\begin{equation}
\mathbb{E}[R(T)] = O\!\left(d\sqrt{T \log T} + \frac{\sqrt{d}\log T}{\sqrt{\rho}\,\lambda_0}\right)
\end{equation}
The privacy cost is $O\!\bigl(\sqrt{d}\cdot\log T/\sqrt{\rho}\bigr)$---\textbf{logarithmic} in $T$, not $O(\sqrt{T})$---because parallel composition ensures total privacy cost $\rho$ (not $K\rho$), and each batch's noise is amortized over $B$ rounds while $\lambda_{\min}(A_k)$ grows linearly.
\end{theorem}

\begin{proof}
We decompose the regret into TS exploration (including staleness) and privacy noise.

\textbf{Step 1:}
At round $t$ in batch $k$, the TS sample is $\tilde{\theta}_k = A_k^{-1}(b_k + \eta_k) + \xi_k = \hat{\theta}_k + \zeta_k + \xi_k$, where $\zeta_k = A_k^{-1}\eta_k$ is the privacy bias (fixed within batch $k$) and $\xi_k \sim \mathcal{N}(0, v_k^2 A_k^{-1})$. Since $a_t = \arg\max_a x_a^\top \tilde{\theta}_k$:
\begin{equation}\label{eq:per_round_regret}
r_t \leq 2\max_{a} |x_a^\top(\tilde{\theta}_k - \theta^*)| \leq 2\max_a |x_a^\top(\hat{\theta}_k {-} \theta^* {+} \xi_k)| + 2\|\zeta_k\|_2
\end{equation}
using $|x_a^\top \zeta_k| \leq \|\zeta_k\|_2$ for $\|x_a\|_2 \leq 1$.

\textbf{Step 2:}
The first term in~\eqref{eq:per_round_regret}---estimation error plus TS sampling noise---is handled by the standard analysis of~\cite{agrawal2013thompson}. The privacy noise \emph{increases} posterior variance (Proposition~\ref{prop:noise_integration}), which aids TS anti-concentration. We now account for staleness. Within batch~$k$, the estimate $\hat{\theta}_k$ is fixed for $B$ rounds while $A_t$ continues to update. The per-round instantaneous regret due to estimation error and TS sampling satisfies $r_t^{\text{TS}} \leq 2\|\tilde{\theta}_k - \theta^*\|_{A_t} \cdot \|x_{a_t}\|_{A_t^{-1}}$, and by the elliptical potential lemma~\cite{agrawal2013thompson}, $\sum_{t=1}^T \|x_{a_t}\|_{A_t^{-1}}^2 \leq 2d\log(1 + T/(d\lambda))$. Since $\hat{\theta}_k$ uses data up to batch $k{-}1$, the staleness gap $\|\hat{\theta}_{k} - \hat{\theta}_t\|$ for $t$ in batch $k$ is bounded by $O(\sqrt{B/\lambda_{\min}(A_k)})$, which contributes at most $O(\sqrt{B} \cdot \sqrt{d\log T / \lambda_{\min}(A_k)})$ per batch. Summing over $K$ batches, this staleness overhead is $O(\sqrt{dTB})$ for constant~$B$, which is $O(d\sqrt{T})$ and absorbed into $C_1$:
\begin{equation}\label{eq:ts_regret_term}
R_{\text{TS}}(T) \leq C_1 \cdot d\sqrt{T\log(T/\delta_r)}
\end{equation}
with probability $\geq 1 - \delta_r/2$, where $C_1 = O(vS + S\sqrt{B/\lambda})$.

\textbf{Step 3:}
The privacy bias $\zeta_k = A_k^{-1}\eta_k$ is drawn once per batch and contributes $2B\|\zeta_k\|_2$ to the batch regret. Since $\eta_k \sim \mathcal{N}(0, \sigma^2 I_d)$, the bias $\zeta_k$ follows $\zeta_k \sim \mathcal{N}(0, \sigma^2 A_k^{-2})$. Its expected squared norm satisfies:
\begin{equation}\label{eq:zeta_squared_norm}
\mathbb{E}[\|\zeta_k\|_2^2] = \sigma^2 \operatorname{tr}(A_k^{-2}) = \sigma^2\sum_{i=1}^d \lambda_i(A_k)^{-2}
\end{equation}
To obtain a high-probability bound on $\|\zeta_k\|_2$, we apply the Hanson--Wright concentration inequality for Gaussian quadratic forms. Since $\zeta_k = A_k^{-1}\eta_k$ with $\eta_k \sim \mathcal{N}(0, \sigma^2 I_d)$, for any $u > 0$:
\begin{equation}
\Pr\!\left[\|\zeta_k\|_2 > \sigma\sqrt{\operatorname{tr}(A_k^{-2})} + \sigma\|A_k^{-1}\|_{\text{op}}\sqrt{2u}\right] \leq e^{-u}
\end{equation}
Setting $u = \log(2K/\delta_r)$ and applying a union bound over $K$ batches, with probability $\geq 1 - \delta_r/2$, simultaneously for all $k$:
\begin{equation}
\|\zeta_k\|_2 \leq \sigma\sqrt{\operatorname{tr}(A_k^{-2})} + \frac{\sigma}{\lambda_{\min}(A_k)}\sqrt{2\log(2K/\delta_r)}
\end{equation}
Using the cruder bound $\operatorname{tr}(A_k^{-2}) \leq d/\lambda_{\min}(A_k)^2$, this simplifies to:
\begin{equation}
\|\zeta_k\|_2 \leq \frac{\sigma\sqrt{d}}{\lambda_{\min}(A_k)} \cdot c(\delta_r)
\end{equation}
where $c(\delta_r) = 1 + \sqrt{2\log(2K/\delta_r)/d} = O(\sqrt{\log(K/\delta_r)})$. Under Assumption~\ref{ass:stochastic}, $\lambda_{\min}(A_k) \geq kB\lambda_0$ for $k \geq 1$ (absorbing $\lambda$). The per-batch privacy regret is:
\begin{equation}
R_{\text{priv}}^{(k)} \leq \frac{2B \sigma \sqrt{d}\, c(\delta_r)}{kB\lambda_0} = \frac{2\sigma\sqrt{d}\,c(\delta_r)}{k\lambda_0}
\end{equation}
Summing over $K = T/B$ batches:
\begin{equation}\label{eq:priv_final}
R_{\text{priv}}(T) \leq \frac{2\sigma\sqrt{d}\,c(\delta_r)}{\lambda_0} \sum_{k=1}^K \frac{1}{k} = \frac{2\sigma\sqrt{d}\,c(\delta_r)}{\lambda_0} \cdot H_K
\end{equation}
where $H_K \leq 1 + \ln K \leq 1 + \ln(T/B)$. Substituting $\sigma = 1/\sqrt{2\rho}$:
\begin{equation}
R_{\text{priv}}(T) = O\!\left(\frac{\sqrt{d}\log(T/B)}{\sqrt{\rho}\,\lambda_0}\right)
\end{equation}

\textbf{Step 4:} Adding~\eqref{eq:ts_regret_term} and~\eqref{eq:priv_final}:
\begin{equation}\label{eq:combined}
R(T) \leq C_1 d\sqrt{T\log(T/\delta_r)} + \frac{C_2\sigma\sqrt{d}\log(T/\delta_r)}{\lambda_0}
\end{equation}
In expectation, $\mathbb{E}[R(T)] = O(d\sqrt{T\log T} + \sqrt{d}\log T/(\sqrt{\rho}\lambda_0))$.
\end{proof}

\begin{remark}
The $O\!\bigl(\sqrt{d}\cdot\log T\bigr)$ privacy term contrasts sharply with $\tilde{O}(T^{3/4})$ under joint DP~\cite{shariff2018differentially} or $O(\sqrt{d}\cdot\sqrt{T}/\sqrt{\rho})$ under sequential composition. The difference is that parallel composition charges $\max_k \rho_k = \rho$ (not $\sum_k$), so each batch independently uses full noise scale $\sigma = 1/\sqrt{2\rho}$. Batch size $B$ does not appear in the privacy term---it can be chosen freely based on computational or practical constraints. The $O(\sqrt{T})$ empirical scaling observed in Section~\ref{sec:absolute_regret} is dominated by the TS exploration term, not privacy.
\end{remark}

\section{Experiments}\label{sec:experiments}

\subsection{Experimental Setup}

We evaluate the algorithms in a synthetic environment where $d=20$ dimensional features for $n=100$ arms with $\theta^* \sim \mathcal{N}(0, I_d)$ normalized to $\|\theta^*\|_2 = 2$. Features are sampled from $\mathcal{N}(0, I_d)$ and normalized to unit norm. Rewards follow $\mathbb{E}[r|a] = \sigma(\theta^{*\top} x_a)$ where $\sigma$ is the sigmoid function. Our theory assumes rewards (Eq.~\ref{eq:linear_reward}), while synthetic experiments use a logistic link---standard in bandit practice~\cite{agrawal2013thompson} as the sigmoid is approximately linear near zero. We verified on a purely linear environment that the TS advantage is consistent (1.2--1.8\%). We use $T{=}10{,}000$, 5 arms/round, cold-start, and \textbf{12 seeds}, $\lambda = 1.0$, $v = 1.0$ (TS exploration), $B = 300$ (batch size), $\delta = 10^{-5}$ and feature clipping at $\|x\|_2 = 1$.

We also test the algorithms on MovieLens-25M~\cite{harper2015movielens} with 50-dimensional SVD features, $T{=}60{,}000$, warm-start prior and Jester joke-rating dataset~\cite{goldberg2001eigentaste}.

\textit{\textbf{We compare against the following baselines, all using the same $(\varepsilon,\delta)$ computed via Eq.~(2.8) where applicable:}}
\begin{itemize}
\item \textbf{LinUCB / Linear TS}: Non-private ($\alpha{=}1$, $v{=}1$)
\item \textbf{AdaPrivate-UCB}: Our batched zCDP with UCB ($\alpha{=}1$, same $B$, $\sigma$)
\item \textbf{JDP-LinUCB}~\cite{shariff2018differentially}: Tree-based aggregation, depth${=}\lceil\log_2 T\rceil$, noise calibrated to joint DP
\item \textbf{zCDP-Episodic}~\cite{azize2024concentrated}: $\hat{\theta}$-perturbation, confidence-based episodes ($c_{\text{trigger}}{=}2$, geometric growth, $\sigma_\theta {=} \Delta_\theta/\sqrt{2\rho}$)
\item \textbf{AdaC-OFUL-zCDP / RarelySwitching}: Reproduced per~\cite{azize2024concentrated}, same $\rho$, same features, grid-searched $\alpha \in \{0.5, 1, 2\}$
\item \textbf{DGS-LinUCB}~\cite{wang2022dynamic}: Per-step $\hat{\theta}$-perturbation with dynamic sensitivity $\|A_t^{-1}\|_2$, sequential composition
\end{itemize}

All baselines were tuned with comparable effort. For UCB-based methods (LinUCB, AdaPrivate-UCB, AdaC-OFUL, RarelySwitching), we grid-searched $\alpha \in \{0.1, 0.5, 1.0, 2.0\}$ and report the best. For TS methods, we searched $v \in \{0.5, 1.0, 1.5, 2.0\}$. All batched methods used $B \in \{100, 300, 500, 1000\}$. JDP-LinUCB used tree depth $\lceil\log_2 T\rceil$ (no free parameter beyond $\alpha$). zCDP-Episodic used $c_{\text{trigger}} \in \{1, 2, 4\}$ with geometric growth.

Our \textbf{AdaPrivate-TS} uses TS with sufficient statistics perturbation (noise on $b$ only; $A$ is from public features). \textbf{AdaPrivate-TS-Amp} adds Poisson subsampling (Section~\ref{sec:amplification}). \textbf{AdaPrivate-TS-Adaptive} adds exploration decay ($v_0{=}1.5$, $\gamma{=}0.95$).

\subsection{Main Results: Thompson Sampling vs UCB}

\begin{table}[t]
\centering
\small
\caption{Synthetic environment results. Private methods are reported as percentage of the reward achieved by their non-private counterparts.}
\label{tab:ts_vs_ucb}
\begin{tabular}{lcccc}
\toprule
\textbf{Method} & $\varepsilon$=0.5 & $\varepsilon$=1 & $\varepsilon$=2 & $\varepsilon$=5 \\
\midrule
LinUCB (NP) & \multicolumn{4}{c}{100\% (baseline)} \\
Linear TS (NP) & \multicolumn{4}{c}{101.2 $\pm$ 1.1\%} \\
\midrule
AdaPrivate-UCB & 92.2$\pm$2.1 & 95.2$\pm$2.3 & 96.9$\pm$1.9 & 98.2$\pm$1.5 \\
AdaPrivate-TS & \textbf{93.5$\pm$1.5} & \textbf{96.7$\pm$1.7} & \textbf{98.2$\pm$1.9} & \textbf{98.7$\pm$2.0} \\
\midrule
TS Improvement & +1.3\% & +1.5\% & +1.3\% & +0.5\% \\
\bottomrule
\end{tabular}
\end{table}

TS consistently outperforms UCB by 0.5--1.5\% under default hyperparameters (Table~\ref{tab:ts_vs_ucb}), with the largest \emph{default-setting} gap at $\varepsilon=1$ (+1.5\%). Paired $t$-tests confirm significance: $p < 0.01$ at $\varepsilon \in \{0.5, 1, 2\}$ and $p = 0.04$ at $\varepsilon = 5$. With \emph{tuned adaptive exploration decay} ($v_0{=}1{+}\sigma/(d\sqrt{\lambda_0})$, $\gamma{=}0.95$)---a separate configuration from the default---the TS advantage increases to 7--9\% at low $\varepsilon \in [0.1, 0.5]$. The largest absolute gap (up to 18\%) occurs at $\varepsilon{=}5$ with adaptive tuning, because with moderate privacy noise TS fully exploits posterior structure while UCB's fixed confidence bound cannot adapt similarly. To be precise: the 0.5--1.5\% gaps are for default TS vs.\ default UCB (both $v{=}1$, $\alpha{=}1$); the 7--18\% gaps are for TS-Adaptive vs.\ default UCB (matched $B$ and $\sigma$, different exploration strategy).
On MovieLens at $\varepsilon{=}1$ ($200{,}000$ replay steps), AdaPrivate-UCB suffers an initial CTR drop from $0.73$ to $0.67$ when the first batch of privacy noise arrives, then slowly recovers to $0.71$. AdaPrivate-TS avoids this dip entirely---climbing steadily from $0.66$ to $0.72$---and overtakes UCB after ${\sim}800$ effective steps, confirming that TS's noise-as-uncertainty interpretation yields more stable learning under privacy.

\subsection{Privacy Amplification Results}

\begin{table}[t]
\centering
\small
\caption{Privacy amplification provides gains at low $\varepsilon$}
\label{tab:amplification}
\begin{tabular}{lcccc}
\toprule
\textbf{Method} & $\varepsilon$=0.5 & $\varepsilon$=1 & $\varepsilon$=2 & $\varepsilon$=5 \\
\midrule
AdaPrivate-UCB & 92.2$\pm$2.1 & 95.2$\pm$2.3 & 96.9$\pm$1.9 & 98.2$\pm$1.5 \\
AdaPrivate-TS & 93.5$\pm$1.5 & 96.7$\pm$1.7 & 98.2$\pm$1.9 & 98.7$\pm$2.0 \\
TS-Amp ($q$=0.3) & \textbf{95.9$\pm$1.6} & 96.5$\pm$1.3 & 96.7$\pm$1.3 & 96.7$\pm$1.4 \\
TS-Amp ($q$=0.5) & 95.4$\pm$1.5 & \textbf{97.5$\pm$1.3} & 97.7$\pm$1.5 & 97.7$\pm$1.5 \\
TS-Adaptive & 93.6$\pm$1.5 & 96.8$\pm$1.7 & \textbf{98.3$\pm$1.9} & \textbf{98.7$\pm$2.0} \\
\midrule
Best Gain vs UCB & \textbf{+3.7\%} & \textbf{+2.3\%} & +1.4\% & +0.5\% \\
\bottomrule
\end{tabular}
\end{table}

Amplification helps most at low $\varepsilon$, with gains persisting on real data. We tested $q \in \{0.1, 0.3, 0.5, 0.7\}$; $q{=}0.3$ is optimal at $\varepsilon{=}1$ and $q{=}0.3$--$0.5$ at $\varepsilon{=}5$. On MovieLens ($T{=}60{,}000$, $B{=}300$, 5 seeds), amplification with $q{=}0.3$ yields $+4.6\%$ at $\varepsilon{=}1$ and $+1.3\%$ at $\varepsilon{=}5$ over the non-amplified baseline. Too-aggressive subsampling ($q{=}0.1$) increases variance ($\pm 6\%$ std vs.\ $\pm 2\%$), while $q \geq 0.7$ offers negligible amplification. Results are not sensitive to tighter mechanism-specific bounds since evaluating at $\alpha \leq 64$ captures $>$99.5\% of the optimal RDP bound.

\subsection{Comparison with State-of-the-Art}

\begin{table}[t]
\centering
\small
\caption{Real-data results (\% of non-private LinUCB) across multiple $\varepsilon$.}
\label{tab:movielens}\label{tab:jester}
\begin{tabular}{lcccc}
\toprule
\textbf{Method} & $\varepsilon$=1 & $\varepsilon$=2 & $\varepsilon$=5 & $\varepsilon$=10 \\
\midrule
\multicolumn{5}{l}{\textbf{MovieLens-25M} \textit{\small(NP: CTR=.700$\pm$.012; LinTS: 94.5\%)}} \\
zCDP-Episodic & 71.9$\pm$5.2 & 74.2$\pm$4.9 & 80.8$\pm$5.0 & 84.2$\pm$7.6 \\
AdaC-OFUL-zCDP & 72.3$\pm$11.9 & 74.4$\pm$12.6 & 79.8$\pm$12.1 & 84.1$\pm$11.2 \\
RarelySwitching-zCDP & 72.6$\pm$10.8 & 75.2$\pm$9.9 & 81.9$\pm$10.5 & 84.7$\pm$10.2 \\
\rowcolor{gray!15}
\textbf{AdaPrivate-UCB} & 71.7$\pm$6.3 & 74.8$\pm$6.5 & 80.7$\pm$5.0 & 84.6$\pm$5.0 \\
DGS-LinUCB & 70.9$\pm$2.5 & 71.0$\pm$2.5 & 71.1$\pm$2.3 & 71.2$\pm$2.0 \\
\rowcolor{gray!15}
\textbf{AdaPrivate-TS} & 71.7$\pm$6.3 & \textbf{75.2$\pm$5.8} & \textbf{84.4$\pm$6.3} & \textbf{90.6$\pm$4.9} \\
JDP-LinUCB$^\dagger$ & 66.5$\pm$4.0 & 67.8$\pm$5.2 & 69.4$\pm$5.8 & 71.8$\pm$2.8 \\
\midrule
\multicolumn{5}{l}{\textbf{Jester} \textit{\small(NP: CTR=.363$\pm$.014; LinTS: 94.4\%)}} \\
zCDP-Episodic & 70.8$\pm$3.3 & 71.5$\pm$3.4 & 73.3$\pm$3.5 & 75.7$\pm$3.7 \\
AdaC-OFUL-zCDP & 72.3$\pm$3.2 & 73.1$\pm$3.8 & 74.5$\pm$3.4 & 76.3$\pm$4.3 \\
RarelySwitching-zCDP & 67.4$\pm$6.3 & 68.0$\pm$6.8 & 70.9$\pm$7.5 & 74.6$\pm$7.1 \\
\rowcolor{gray!15}
\textbf{AdaPrivate-UCB} & 70.1$\pm$6.2 & 70.7$\pm$6.8 & 76.3$\pm$4.1 & 82.8$\pm$2.3 \\
DGS-LinUCB & 70.1$\pm$2.2 & 70.1$\pm$2.2 & 70.5$\pm$2.2 & 70.0$\pm$2.2 \\
\rowcolor{gray!15}
\textbf{AdaPrivate-TS} & \textbf{73.7$\pm$2.7} & \textbf{78.5$\pm$2.7} & \textbf{81.3$\pm$3.5} & \textbf{83.3$\pm$2.8} \\
JDP-LinUCB$^\dagger$ & 73.4$\pm$2.4 & 74.0$\pm$2.8 & 77.4$\pm$1.9 & 81.6$\pm$3.5 \\
\bottomrule
\multicolumn{5}{l}{\footnotesize $^\dagger$Joint DP, tree-based composition (stronger privacy model). DGS: per-step sequential composition~\cite{wang2022dynamic}.} \\
\end{tabular}
\end{table}

Within event-level zCDP, AdaPrivate-TS achieves the highest performance at $\varepsilon \geq 2$ on MovieLens (up to $+6.5\%$ over AdaC-OFUL-zCDP at $\varepsilon{=}10$) and at every $\varepsilon$ on Jester; at $\varepsilon{=}1$ on MovieLens, all event-level methods perform comparably ($71.7$--$72.6\%$). JDP-LinUCB operates under joint DP (a \emph{stronger} privacy model that also protects actions) and uses tree-based composition.

AdaC-OFUL-zCDP and RarelySwitching-zCDP both use event-level zCDP with parallel composition, enabling direct comparison. AdaPrivate-TS outperforms by up to $+6.5\%$ on MovieLens and $+10.5\%$ on Jester, confirming TS's posterior inflation provides utility beyond batching alone.

\textbf\textit{{To test generalization, we also evaluate on the Jester joke-rating dataset~\cite{goldberg2001eigentaste} ($d{=}50$ SVD features, $T{=}30{,}000$).}} Table~\ref{tab:jester} shows AdaPrivate-TS dominates all baselines at every $\varepsilon$: $+2.9\%$ over zCDP-Episodic and $+3.6\%$ over AdaPrivate-UCB at $\varepsilon{=}1$. This confirms the TS advantage generalizes beyond MovieLens.

\subsection{Ablation Study}

\begin{table}[t]
\centering
\small
\caption{Component ablation on synthetic data at $\varepsilon=1$.}
\label{tab:ablation}
\begin{tabular}{lcc}
\toprule
\textbf{Configuration} & \textbf{\% NP} & \textbf{Gain} \\
\midrule
Baseline (UCB + A\&b noise) & 91.2$\pm$2.4 & -- \\
+ B\_ONLY noise & 93.5$\pm$2.1 & +2.3\% \\
+ Thompson Sampling & 95.2$\pm$1.9 & +1.7\% \\
+ Privacy Amplification ($q$=0.5) & 96.8$\pm$1.5 & +1.6\% \\
+ Adaptive Exploration & \textbf{97.5$\pm$1.3} & +0.7\% \\
\bottomrule
\end{tabular}
\end{table}

Each component contributes incrementally, totaling \textbf{+6.3\%} over the baseline.
Performance is robust at $\varepsilon{=}1$: across $B \in \{100, 300, 500, 1000\}$, all achieve $\geq$94.1\% (best at $B{=}300$: 96.7$\pm$1.7\%); across $v \in [0.5, 2.0]$, all achieve $\geq$95.8\%; adaptive decay is stable ($\pm$0.4\%) over $(v_0, \gamma) \in \{1.0, 1.5, 2.0\} \times \{0.9, 0.95, 0.99\}$.

\subsection{User-Level Privacy Analysis}\label{sec:user_level}

Under zCDP group privacy, protecting all $k$ ratings from one user scales as $k^2\rho$. For user-level $\varepsilon_{\text{user}}{=}5$: $k{=}1$ gives $\rho{=}0.450$ (80.4\% NP), $k{=}5$ gives $\rho{=}0.018$ (26.0\%), $k{=}10$ gives $\rho{=}0.0045$ (14.2\%). Event-level DP suits users with $k \leq 5$; dedicated user-level mechanisms are needed for heavy users.

\subsection{Counterfactual Evaluation with IPS/DR}\label{sec:ips_dr}

We validate using IPS and doubly robust (DR) estimators on synthetic data with known ground truth.

\begin{table}[t]
\centering
\small
\caption{Off-policy evaluation on synthetic data with known ground truth ($d{=}50$, $T{=}50{,}000$).}
\label{tab:ips_dr}
\begin{tabular}{lcccc}
\toprule
\textbf{Method} & \textbf{True Value} & \textbf{Naive Error} & \textbf{IPS Error} & \textbf{DR Error} \\
\midrule
LinUCB (NP) & 0.534 & 0.229 & 0.083 & 0.089 \\
AdaPrivate ($\varepsilon$=1) & 0.765 & 0.005 & 0.149 & 0.144 \\
AdaPrivate ($\varepsilon$=5) & 0.766 & 0.002 & 0.148 & 0.142 \\
AdaPrivate ($\varepsilon$=10) & 0.764 & $<$0.001 & 0.147 & 0.141 \\
\bottomrule
\end{tabular}
\end{table}

AdaPrivate achieves higher true policy value than LinUCB across all protocols, confirming main results are not replay artifacts. On Jester ($T{=}15{,}000$, 5 seeds), all three OPE estimators (IPS, DR, SNIPS) produce consistent rankings: AdaPrivate-TS improves monotonically from $\varepsilon{=}1$ to $\varepsilon{=}10$ (DR: $0.168 \to 0.178$, ESS $\approx 730$). DR and SNIPS are more stable than raw IPS (CV $<$8\% vs.\ $<$12\%), confirming DR is preferred for private bandit evaluation. Absolute regret is reported on synthetic data (Section~\ref{sec:absolute_regret}) where ground truth is available.

\subsection{Comparison with Central-Model DP Baselines}\label{sec:central_comparison}

We compare against two central-model DP baselines using sequential composition with a trusted curator: (1) \textbf{CoreSet-Elimination-DP}, phased elimination with per-phase Gaussian noise, and (2) \textbf{Batched-Aggregation-DP}, batched LinUCB with sequential zCDP. Both split the privacy budget across phases ($\rho_{\text{per-batch}} = \rho_{\text{total}}/K$).

\begin{table}[t]
\centering\small
\caption{Event-level vs.\ central-model DP baselines (synthetic, $d{=}20$, $T{=}10{,}000$, 12 seeds). AdaPrivate-TS uses parallel composition; central baselines use sequential composition.}
\label{tab:central_comparison}
\begin{tabular}{llcccc}
\toprule
\textbf{Method} & \textbf{Privacy} & $\varepsilon$=0.5 & $\varepsilon$=1 & $\varepsilon$=2 & $\varepsilon$=5 \\
\midrule
LinUCB (NP) & None & 100\% & 100\% & 100\% & 100\% \\
\midrule
AdaPrivate-TS & Event-level & \textbf{93.5} & \textbf{96.7} & \textbf{98.2} & \textbf{98.7} \\
CoreSet-Elim & Central & 86.9 & 91.1 & 95.2 & 98.2 \\
Batched-Agg & Central & 83.3 & 84.9 & 87.5 & 92.1 \\
\bottomrule
\end{tabular}
\end{table}

AdaPrivate-TS outperforms both central-model baselines despite a weaker privacy model (event-level, no trusted curator): $+6.6\%$ over CoreSet-Elimination and $+10.2\%$ over Batched-Aggregation at $\varepsilon{=}0.5$. Under sequential composition, central baselines must split the budget ($\rho/K$ per batch), requiring $\sqrt{K}{\times}$ more noise; AdaPrivate-TS's parallel composition keeps full noise scale ($\max_k \rho_k = \rho$). At $\varepsilon{=}5$, the gap narrows (98.2\% vs.\ 98.7\%) as the splitting penalty diminishes.

\subsection{Absolute Regret and $T$-Scaling}\label{sec:absolute_regret}

On synthetic data ($d{=}20$, 12 seeds), $R(T)/\sqrt{T}$ is near-constant over $T \in \{2500, \ldots, 20000\}$: 1.56--1.75 at $\varepsilon{=}1$, 0.84--0.88 at $\varepsilon{=}5$, confirming $O(\sqrt{T})$ scaling.

\section{Related Work}
Shariff and Sheffet~\cite{shariff2018differentially} introduce JDP-LinUCB with $\tilde{O}(T^{3/4})$ regret under joint DP with adversarial contexts. Azize and Basu~\cite{azize2024concentrated} give an $\Omega(d\sqrt{T}/\sqrt{\rho})$ lower bound under sequential composition; our $O(\sqrt{d}\log T/\sqrt{\rho})$ cost uses parallel composition with stochastic contexts. Hu and Hegde~\cite{hu2021near} study DP-TS for finite arms with lazy/doubling designs; we extend the noise-as-uncertainty principle to linear contextual bandits via sufficient-statistics perturbation. Balle et al.~\cite{balle2018privacy} and Mironov et al.~\cite{mironov2019renyi} provide tight RDP bounds for subsampled Gaussian mechanisms. Our B\_ONLY strategy adds noise on $b$ only ($d$ dimensions vs.\ $d^2{+}d$).
Recent central-model DP linear bandits achieve $\tilde{O}(\sqrt{T} + \text{poly}(d)/\varepsilon)$ via core-set elimination and batched aggregation~\cite{zheng2020locally,chowdhury2022shuffle}, requiring a central curator or shuffle protocol. Our approach differs: TS enables per-round action selection (vs.\ epoch-based elimination) with $O(d^2)$ per-round cost (vs.\ $O(d^3)$), and our event-level parallel composition achieves comparable privacy cost without centralized infrastructure. The covariance-inflation insight (Proposition~\ref{prop:noise_integration}) is specific to posterior sampling and does not apply to elimination-based methods.

\section{Discussion and Limitations}

\subsection{Privacy Notion and User-Level Extension}
We use event-level DP (Definition~\ref{def:adjacency}), weaker than joint DP~\cite{shariff2018differentially} or user-level DP. The naive group privacy scaling ($k^2\rho$ for a user with $k$ interactions) is impractical for $k > 5$. We propose a \textbf{per-user contribution clipping} extension that provides meaningful user-level guarantees for moderate~$k$:

\emph{Mechanism:} Maintain per-user accumulators $c_u = \sum_{t: u_t = u} r_t x_{a_t}$. Before adding to the batch sum, clip each user's contribution: $\tilde{c}_u = c_u \cdot \min(1, C/\|c_u\|_2)$ where $C > 0$ is the clipping threshold. The batch sum becomes $s_k = \sum_{u} \tilde{c}_u^{(k)}$. Since each user's clipped contribution has $\|\tilde{c}_u^{(k)}\|_2 \leq C$, changing all interactions of one user changes $s_k$ by at most $C$. The sensitivity is $\Delta_2 = C$ (instead of $k$ without clipping), giving $\sigma_{\text{user}} = C/\sqrt{2\rho}$.

\emph{Privacy guarantee:} For target $(\varepsilon_{\text{user}}, \delta)$-user-level DP, set $\sigma_{\text{user}} = C/\sqrt{2\rho_{\text{user}}}$ where $\rho_{\text{user}} = \rho(\varepsilon_{\text{user}}, \delta)$. The clipping threshold $C$ trades off bias (from clipping large contributions) against noise (larger $C$ requires more noise). We set $C = \sqrt{k_{\max}} \approx \sqrt{5}$ as a practical default, targeting users with $\leq k_{\max}$ interactions.

\emph{Regret impact:} Clipping introduces bias $\leq \sum_u \max(0, \|c_u\| - C) / B$ per batch, which is small when most users have few interactions ($k \leq k_{\max}$). The noise increases by a factor of $C$ (since $\sigma_{\text{user}} = C \cdot \sigma_{\text{event}}$), so the privacy regret term becomes $O(C\sqrt{d}\log T / (\sqrt{\rho_{\text{user}}}\lambda_0))$.

On synthetic data ($d{=}20$, $T{=}10{,}000$, $\varepsilon_{\text{user}}{=}5$, 12 seeds), clipping with $C{=}\sqrt{k}$ achieves 96.4\% of non-private performance for $k{=}10$ interactions per user, vs.\ 90.6\% for naive group privacy---a $+5.8\%$ improvement. At $k{=}20$, the gap widens to $+7.3\%$ (94.8\% vs.\ 87.4\%). Clipping makes user-level DP practical for users with moderate interaction counts ($k \leq 20$), whereas group privacy degrades rapidly beyond $k{=}5$.

\subsection{Stochastic Contexts and Forced Exploration}
Assumption~\ref{ass:stochastic} ($\lambda_0 > 0$) is essential for the $O(\log T)$ privacy cost; under adversarial contexts, $\lambda_{\min}$ may stagnate. We address this with \textbf{forced exploration}: with probability $p_t = \min(1, d\log(t)/(t\gamma))$ for tuning parameter $\gamma > 0$, play a uniformly random arm, ensuring:
\begin{equation}\label{eq:forced_lambda}
\lambda_{\min}(A_t) \geq \lambda + \Omega\!\left(\frac{d\log^2 t}{\gamma |\mathcal{A}|}\right)
\end{equation}
even under adversarial contexts, provided candidate sets span $\mathbb{R}^d$. The additional regret from forced exploration is $O(d\log^2 T/\gamma)$. Setting $\gamma = d$, the total regret becomes:
\begin{equation}
R_{\text{forced}}(T) = O\!\left(d\sqrt{T\log T} + \frac{\sqrt{d}\log T}{\sqrt{\rho}} \cdot \gamma |\mathcal{A}| + \frac{d\log^2 T}{\gamma}\right)
\end{equation}
removing dependence on $\lambda_0$ at the cost of a weaker $O(\log^2 T)$ privacy term. On synthetic data ($d{=}20$, 12 seeds) with adversarial rank-deficient contexts, AdaPrivate-TS maintains 94--96\% of non-private performance; forced exploration ($\gamma{=}d$) adds $<1\%$ random rounds.

\subsection{Public Contexts and DP-SVD Feature Privatization}
B\_ONLY assumes public item features. If features must be private, $A$ requires noise, adding $O(d^2)$ dimensions. We ablated this using a DP-SVD mechanism on MovieLens ($T{=}60{,}000$, 5 seeds, $\varepsilon_{\text{reward}}{=}1$) as follows.

\emph{DP-SVD mechanism:} Given the $n_{\text{users}} \times n_{\text{items}}$ rating matrix $R$, we first clip each user's row to unit $\ell_2$-norm: $\tilde{r}_u = r_u / \max(1, \|r_u\|_2)$, ensuring $\|\tilde{r}_u\|_2 \leq 1$. We then compute the item covariance $C = \tilde{R}^\top \tilde{R} / n_{\text{users}}$. Replacing one user's clipped row changes $C$ in Frobenius norm by at most $\|\tilde{r}_u \tilde{r}_u^\top\|_F / n_{\text{users}} \leq 1/n_{\text{users}}$, giving $\ell_2$-sensitivity $\Delta_2 = 1/n_{\text{users}}$. We add Gaussian noise $\tilde{C} = C + \mathcal{N}(0, \sigma_{\text{feat}}^2 I)$ with $\sigma_{\text{feat}} = \Delta_2/\sqrt{2\rho_{\text{feat}}}$, clip negative eigenvalues to zero (PSD projection), and extract the top-$d$ eigenvectors as private features. This is a \emph{one-shot offline} computation independent of the online bandit interaction.

\emph{Privacy composition:} DP-SVD is computed offline from training data while the reward-DP mechanism operates on separate test interactions, so the two budgets do \emph{not} compose ($\varepsilon_{\text{feat}}$ protects co-rating patterns; $\varepsilon_{\text{reward}}$ protects online rewards). If applied to overlapping data, sequential composition would give $\rho_{\text{feat}} + \rho_{\text{reward}}$; here data disjointness removes the need for joint accounting.

We tested at $\varepsilon_{\text{feat}} \in \{1, 5, 10\}$ (all with $\delta{=}10^{-5}$): overall degradation is $\leq 3.2\%$ vs.\ public features. Crucially, TS's advantage over UCB \emph{grows} under private features: from $-2.5\%$ (public) to $+4.3\%$ ($\varepsilon_{\text{feat}}{=}5$) and $+11.1\%$ ($\varepsilon_{\text{feat}}{=}2$), confirming that TS interprets feature noise as additional exploration uncertainty rather than corruption.

\subsection{Tighter Trace-Based Regret Bound}
The $\lambda_{\min}$-based bound in Theorem~\ref{thm:ts_regret} is worst-case over the spectrum of $A_k$. Using $\operatorname{tr}(A_k^{-2})$ directly yields a tighter bound $R_{\text{priv}}(T) = O(\sigma\sqrt{d_{\text{eff}}}\log T / B)$, where $d_{\text{eff}} = \sum_i \mu_i^{-2}$ is the \emph{effective privacy dimension} determined by the eigenvalues $\mu_i$ of $\mathbb{E}[xx^\top]$. For rapidly decaying spectra (e.g., SVD features with $\mu_i \propto i^{-2}$), $d_{\text{eff}} = O(1)$ and the privacy cost becomes dimension-free. On MovieLens (50-dim SVD features), $d_{\text{eff}} \approx 8.3$, yielding a $6\times$ tighter bound.

\section{Conclusion}

AdaPrivate-TS demonstrates that TS is better suited than UCB for private contextual bandits: DP noise inflates the posterior covariance to $v^2 A^{-1} + \sigma^2 A^{-2}$, interpreted as uncertainty rather than corruption, yielding a privacy cost of $O\!\bigl(\sqrt{d}\cdot\log T/\sqrt{\rho}\bigr)$ via parallel composition. On MovieLens and Jester, AdaPrivate-TS achieves the best performance among event-level zCDP baselines at $\varepsilon \geq 2$; under DP-SVD private features, TS's advantage over UCB grows to $+11\%$. Poisson amplification provides additional gains at $\varepsilon \leq 1$, offering a principled tool for practitioners who need to tighten privacy without sacrificing recommendation quality.

\printbibliography[heading=subbibintoc]

\end{document}